\documentclass[10pt,twocolumn,letterpaper]{article}

\usepackage{wacv}

\wacvfinalcopy % *** Uncomment this line for the final submission

\def\wacvPaperID{119} % *** Enter the wacv Paper ID here

\ifwacvfinal\fi
\usepackage{color}
\usepackage{times}
\usepackage{epsfig}
\usepackage{graphicx}
\usepackage{amsmath}
\usepackage{amssymb}
\usepackage{amsthm}

\usepackage{times}
\usepackage{epsfig}
\usepackage{subfigure}
\usepackage{url}

\usepackage{enumerate}
\usepackage{stfloats}

\usepackage[vlined, ruled]{algorithm2e}
\usepackage{setspace}
\usepackage{caption}
\usepackage{bm}
\usepackage{algorithm2e}
\usepackage{bbold}
\renewcommand\cdots{...}

\newcommand{\bigoh}{\mathcal{O}}

\newcommand{\enorm}[1]{\left\|{#1}\right\|_2}

\newcommand{\simplex}[1]{\Delta^{#1}}
\newcommand{\set}[1]{\left\{#1\right\}}

\newcommand{\suptensor}[1]{\mathfrak{S}^{d}}

\newcommand{\vS}{\boldsymbol{S}}

\newcommand{\comment}[1]{}
\newcommand{\seqset}{\mathcal{P}}

\newtheorem{theorem}{Theorem}
\newtheorem{definition}{Definition}
\newtheorem{proposition}{Proposition}

\newcommand{\reals}[1]{\mathbb{R}^{#1}}
\newcommand{\myexp}[1]{e^{\left({#1}\right)}}

\renewcommand\cdots{...}

\newcommand{\pivotset}{\mathcal{Z}}
\newcommand{\seq}{S}

\newcommand{\labels}{\mathcal{L}}

\newcommand{\allframesset}{\mathcal{F}}
\renewcommand{\seqset}{\mathcal{S}}
\newcommand{\classifier}{P}
\newcommand{\op}[1]{\otimes_{#1}}
\DeclareMathOperator*{\HOK}{HOK}
\DeclareMathOperator*{\HOSVD}{HOSVD}
\DeclareMathOperator*{\sign}{sign}
\newcommand{\vA}{\mathcal{A}}

\newcommand{\hatclassifier}{\hat{P}}
\newcommand{\card}[1]{\left|{#1}\right|}
\begin{document}

%%%%%%%%% TITLE
\title{Higher-order Pooling of CNN Features via\\
Kernel Linearization for Action Recognition}

% Authors at the same institution
\author{
Anoop Cherian${^{1,3}}$ \hspace{2cm} Piotr Koniusz${^{2,3}}$ \hspace{2cm} Stephen Gould${^{1,3}}$ \\
${^1}$ Australian Center for Robotic Vision \hspace{2cm}  ${^2}$Data61/CSIRO, \\
${^3}$ The Australian National University, Canberra, Australia
\\
{\tt\small anoop.cherian@anu.edu.au} \hspace{1cm}
{\tt\small piotr.koniusz@data61.csiro.au} \hspace{1cm}
{\tt\small stephen.gould@anu.edu.au} \hspace{1cm}
}

\maketitle
\ifwacvfinal\thispagestyle{empty}\fi

%%%%%%%%% ABSTRACT
\begin{abstract}
Most successful deep learning algorithms for action recognition extend models designed for image-based tasks such as object recognition to video. Such extensions are typically trained for actions on single video frames or very short clips, and then their predictions from sliding-windows over the video sequence are pooled for recognizing the action at the sequence level. Usually this pooling step uses the first-order statistics of frame-level action predictions. In this paper, we explore the advantages of using higher-order correlations; specifically, we introduce Higher-order Kernel (HOK) descriptors generated from the late fusion of CNN classifier scores from all the frames in a sequence. To generate these descriptors, we use the idea of kernel linearization. Specifically, a similarity kernel matrix, which captures the temporal evolution of deep classifier scores, is first linearized into kernel feature maps. The HOK descriptors are then generated from the higher-order co-occurrences of these feature maps, and are then used as input to a video-level classifier. We provide experiments on two fine-grained action recognition datasets, and show that our scheme leads to state-of-the-art results.

\end{abstract}

%Deep learning models for action recognition are usually trained on short clips consisting of only a few frames. Such short temporal receptive fields may be ineffective, especially when the actions are long, contain several sub-actions, or are subtle as in a fine-grained setting.  In this paper, we hypothesize that while frame-level classifier predictions might be noisy, the higher-order correlations between them from several frames can capture the dynamics of the actions holistically. Towards this end, in this paper, we introduce~\emph{Higher-order Kernel} (HOK) descriptors generated from the late fusion of CNN classifier scores from all the frames in a sequence. To generate these descriptors, we utilize the idea of kernel linearization. That is, a similarity kernel matrix which captures the temporal evolution of deep classifier scores is first linearized into kernel feature maps. The HOK descriptors are then generated from the higher-order co-occurrences of these feature maps. We provide experiments on two fine-grained action recognition datasets, and show that our scheme leads to state-of-the-art results.
\section{Introduction}
\label{sec:intro}
With the resurgence of efficient deep learning algorithms, recent years have seen significant advancements in several fundamental problems in computer vision, including human action recognition in videos~\cite{simonyan2014two,tran2014c3d,karpathy2014large,donahue2014long}. However, despite this progress, solutions for action recognition are far from being practically useful, especially in a real-world setting. This is because real-world actions are often subtle, may use hard-to-detect tools (such as knives, peelers, etc.), may involve strong appearance or human pose variations, may be of different durations, or done at different speeds, among several other complicating factors. In this paper, we study a difficult subset of the general problem of action recognition, namely~\emph{fine-grained action recognition}. This problem category is characterized by actions that have very weak intra-class appearance similarity (such as cutting tomatoes versus cutting cucumbers), while strong inter-class similarity (peeling cucumbers versus slicing cucumbers). Figure~\ref{fig:illustration} illustrates two fine-grained actions. 

\begin{figure}
	\centering   
	\includegraphics[width=3.5cm]{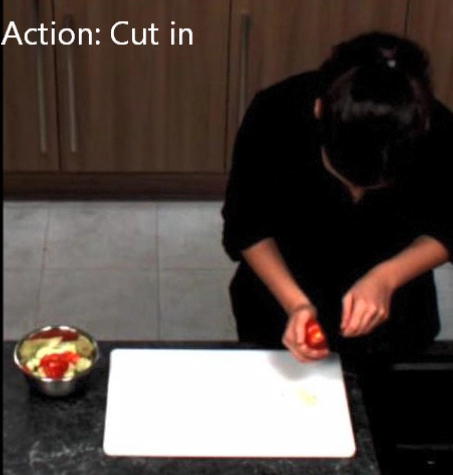}\hspace{0.3cm}
	\includegraphics[width=3.5cm]{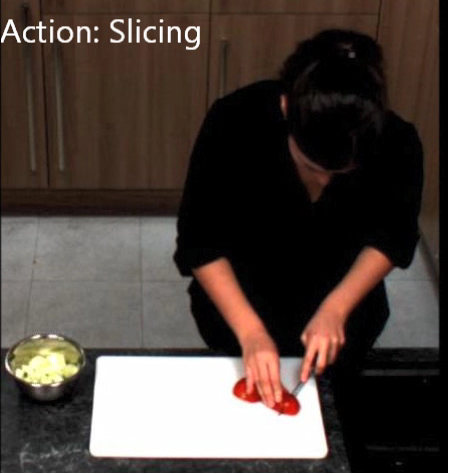}
	\caption{Fine-grained action instances from two different action categories: \emph{cut-in} (left) and \emph{slicing} (right). These instances are from the MPII cooking activities dataset~\cite{rohrbach2012database}.}
     \label{fig:illustration}
\end{figure}

Unsurprisingly, the most recent trend for fine-grained activity recognition is based on convolutional neural networks (CNN)~\cite{cheron2015p,ji20133d}. These schemes extend frameworks developed for general purpose action recognition~\cite{simonyan2014two} into the fine-grained setting by incorporating heuristics that extract auxiliary discriminative cues, such as the position of body parts or indicators of human-to-object interaction. A technical difficulty when extending CNN-based methods to videos is that unlike objects in images, the actions in videos are spread across several frames. Thus to correctly infer actions a CNN must be trained on the entire video sequence. However, the current computational architectures are prohibitive in using more than a few tens of frames, thus limiting the size of the temporal receptive fields. For example, the two stream model~\cite{simonyan2014two} uses single frames or a tiny sets of optical flow images for learning the actions. One may overcome this difficulty by switching to  recurrent networks~\cite{baccouche2011sequential} which typically require large training sets for effective learning.

As is clear, using single frames to train CNNs might be insufficient to capture the dynamics of actions effectively, while a large stack of frames requires a larger number of CNN parameters that result in model overfitting, thereby demanding larger training sets and computational resources. This problem also exists in other popular CNN architectures such as 3D CNNs~\cite{tran2014c3d,ji20133d}. Thus, state-of-the-art deep action recognition models are usually trained to generate useful features from short video clips that are then pooled to generate holistic sequence level descriptors, which are then used to train a linear classifier on action labels. For example, in the two-stream model ~\cite{simonyan2014two}, the soft-max scores from the final CNN layers from the RGB and optical flow streams are combined using average pooling. Note that average pooling captures only the first-order correlations between the scores; a higher-order pooling \cite{me_tensor} that captures higher-order correlations between the CNN features can be more appropriate, which is the main motivation for the scheme proposed in this paper.

Specifically, we assume a two-stream CNN framework as suggested in~\cite{simonyan2014two} with separate RGB image and optical flow streams. Each of these streams is trained on single RGB or optical flow frames from the sequences against the action labels. As noted above, while CNN classifier predictions at the frame-level might be very noisy, we posit that the correlations between the temporal evolution of classifier scores can capture useful action cues which may help improve recognition performance. Intuitively, some of the actions might have unlabelled pre-cursors (such as ~\emph{Walking},~\emph{Picking}, etc.). Using higher-order action occurrences in a sequence may be able to capture such pre-cursors, leading to better discrimination of the action, while they may be ignored as noise when using a first-order pooling scheme. In this paper, we use this intuition to develop a theoretical framework for action recognition using higher-order pooling on the two-stream classifier scores.

%We also show that using our model can allow for incorporating other prior knowledge into action modeling, such as imposing a temporal soft-grammar on the predictions by localizing the sub-actions in a sequence. 

%To clarify, let us consider two fine-grained activities,~\emph{washing plates} and~\emph{wiping plates}. Suppose the first activity has a pre-cursor in the sequences, such as~\emph{walking to the wash basin}, then the CNN classifiers trained on~\emph{washing plates} sequences will be confused on the frames from this precursor. However, if no such precursor (or different precursor) exists for~\emph{wiping plates}, then computing a higher-order correlation between the classifier scores will be discriminative as it captures the precursor confusions between the classifiers. 

Our pooling scheme is based on kernel linearization---a simple technique that decomposes a similarity kernel matrix computed on the input data in terms of a set of anchor points (pivots). Using a kernel (specifically, a Gaussian kernel) offers richer representational power, e.g., by embedding the data in an infinite dimensional Hilbert space. The use of the kernel also allows us to easily incorporate additional prior knowledge about the actions. For example, we show in Section~\ref{sec:method} how to impose a soft-temporal grammar on the actions in the sequence. Kernels are a class of positive definite objects that belong to a non-linear but Riemannian geometry, and thus directly evaluating them for use in a dual SVM (c.f. primal SVM) is computationally expensive. 
\comment{might demand non-linear classifiers.} Therefore, these kernels are usually linearized into feature maps so that fast linear classifiers can be applied instead. In this paper, we use a linearization technique developed for shift-invariant kernels (Section.~\ref{sec:prelims}). Using this technique, we propose~\emph{Higher-order Kernel} descriptors (Section.~\ref{sec:HOK}) that capture the higher-order co-occurrences of the kernel maps against the pivots. We apply a non-linear operator (such as~\emph{eigenvalue power normalization}~\cite{me_tensor}) to the HOK descriptors as it is known to lead to superior classification performance. The HOK descriptor is then vectorized and used for training actions. Note that the HOK descriptors, that we propose here, belong to the more general family of third-order super-symmetric tensor (TOSST) descriptors introduced in~\cite{sparse_tensor_cvpr,tensor_eccv}.

We provide experiments (Section~\ref{sec:expts}) using the proposed scheme on two standard action datasets, namely (i) the MPII Cooking activities dataset~\cite{rohrbach2012database}, and (ii) the JHMDB dataset~\cite{jhuang2013towards}. Our results demonstrate that higher-order pooling is useful and can perform competitively to the state-of-the-art methods.

\section{Related Work}
\label{sec:related_work}
%The problem of fine-grained action recognition has witnessed significant interest from the computer vision community in the recent years.
Initial approaches to tackle fine-grained action recognition have been direct extensions of schemes developed for the general classification problems, mainly based on hand-crafted features. A few notable such approaches include ~\cite{pishchulin2014fine,rohrbach2012database,rohrbach2015recognizing,wang2013dense}, in which features such as HOG, SIFT, HOF, etc., are first extracted at spatio-temporal interest point locations (e.g., following dense trajectories) and then fused and used to train a classifier. However, the recent trend has moved towards data driven feature learning via deep learning  platforms~\cite{krizhevsky2012imagenet,simonyan2014two,ji20133d,tran2014c3d,donahue2014long,yue2015beyond}. As alluded to earlier, the lack of sufficient annotated video data, and the need for expensive computational infrastructure, makes direct extension of these frameworks (which are primarily developed for image recognition tasks) challenging for video data, thus demanding efficient representations.

%Pose provides an estimate of the body-part configurations of humans involved in the activity and has been shown to aid activity recognition~\cite{yao2011does,zuffi2013puppet}.
%In \emph{Ch\'{e}ron et al.}~\cite{cheron2015p}, human pose is used as prior to select video regions containing actions and then use a two-stream CNN ~\cite{simonyan2014two} at these regions. 

Another promising setup for fine-grained action recognition has been to use mid-level features such as human pose. Clearly, estimating the human pose and developing action recognition systems on it disentangles the action inference from operating directly on pixel-level, thus allowing for higher-level of sophisticated action reasoning~\cite{rohrbach2012database,wang2013approach,zuffi2013puppet,cheron2015p}. Although, there have been significant advancements in the pose estimation recently~\cite{wei2016cpm,Insafutdinov2016}, most of these models are computationally demanding and thus difficult to scale to millions of video frames that form standard datasets. %Moreover, the action might involve human appearances at multiple scales, postures, and occlusions, which all challenge contemporary pose estimation algorithms which are not capable of performing well in these conditions.

A different approach to fine-grained recognition is to detect and analyze human-object interactions in the videos. Such a technique is proposed in \emph{Zhou et al.}~\cite{zhou2015interaction}. Their method starts by generating region proposals for human-object interactions in the scene, extracts visual features from these regions and trains a classifier for action classes on these features. A scheme based on tracking human hands and their interactions with objects is presented in Ni et al.~\cite{ni2014multiple}. Hough forests for action recognition are proposed in Gall et al.~\cite{gall2011hough}. Although recognizing objects may be useful, they may not be easily detectable in the context of fine-grained actions.
%while a similar approach using an unsupervised  probabilistic framework for learning interactions from depth cameras is developed in~\cite{lei2012fine,wu2015watch}.

We also note that there have been several other deep learning models devised for action modeling such as using 3D convolutional filters~\cite{ji20133d}, recurrent neural networks~\cite{baccouche2011sequential}, long-short term memory networks~\cite{donahue2014long}, and large scale video classification architectures~\cite{karpathy2014large}. These models demand huge collections of videos for effective training, which are usually unavailable for fine-grained activity tasks and thus the applicability of these models is yet to be ascertained.

Pooling has been a useful technique for reducing the size of video representations, thereby enabling the applicability of efficient machine learning algorithms to this data modality. Recently, a pooling scheme preserving the temporal order of the frames is proposed by \emph{Fernando et al.}~\cite{Fernando_2015_CVPR} by solving a Rank-SVM formulation. In \emph{Wang et al.}~\cite{wang2015action}, deep features are fused along action trajectories in the video. Correlations between space-time features are proposed in~\cite{shechtman2005space}. Early and late fusion of CNN feature maps for action recognition are proposed in~\cite{karpathy2014large,yue2015beyond}. Our proposed higher-order pooling scheme is somewhat similar to the second- and higher-order pooling approaches proposed in~\cite{carreira2012semantic} and ~\cite{me_tensor}
%proposed in~\cite{carreira2012semantic,me_tensor} 
that generate  
%symmetric positive definite 
representations from low-level descriptors for the task of semantic segmentation of images and object category recognition, respectively. Moreover, our HOK descriptor is inspired by the sequence compatibility kernel (SCK) descriptor introduced in~\cite{tensor_eccv} which pools higher-order occurrences of feature maps from skeletal body joints for action recognition. 
In contrast, we use the frame-level prediction vectors (output of fc8 layers) from the deep classifiers to generate our pooled descriptors, therefore, the size of our pooled descriptors is a function of the  number of action classes. Moreover, unlike SCK, that uses pose skeletons, we use raw action videos directly.
Our work is also different from works such as~\cite{sparse_tensor_cvpr,vasilescu2002multilinear} in which tensor descriptors are proposed on hand-crafted features. In this paper, we show how CNNs could benefit from higher-order pooling for the application of fine-grained action recognition.

\section{Background}
\label{sec:prelims}
In this section, we first review the tensor notation that we use in the following sections. This precedes an introduction to the necessary theory behind kernel linearization on which our descriptor framework is based. 

\subsection{Notation}

Let $a\in\reals{d}$ be a $d$-dimensional vector. Then, we use $\vA=a^{\op{r}}$ to denote the $r$-mode super-symmetric tensor generated by the $r$-th order outer-product of $a$, where the element at the $\left(i_1,i_2,\cdots, i_{r}\right)$-th index is given by $\Pi_{j=1}^r a_{i_j}$. We use the notation $\vA=\vS\times_{k=1}^n P_k$ for matrices $P_k$ and a tensor $\vS$ to denote that $\vA_{(i_1,i_2,\cdots,i_n)}=\sum_{j_1}\sum_{j_2}\cdots\sum_{j_n}\vS_{(i_1,i_2,\cdots,i_n)}P_1^{(i_1,j_1)} P_2^{(i_2,j_2)}\cdots P_n^{(i_n,j_n)}$. This notation arises in the Tucker decomposition of higher-order tensors, see \cite{lathauwer_hosvd, kolda_tensorrew} for details. We note that the inner-product between two such tensors follows the general element-wise product and summation, as is typically used in linear algebra. We assume in the sequel that the order $r$ is ordinal and greater than zero. We use the notation $\langle \cdot , \cdot \rangle$ to denote the standard Euclidean inner product and $\simplex{d}$ for the $d$-dimensional simplex.

%
%Let $\tV\in\mbr{d\times d\times d}$ denote a third-order super-symmetric tensor. Using Matlab style notation, we refer to the $p$-th slice of this tensor as $\tV_{:,:,p}$, which is a $d_1\times d_2$ matrix. For a matrix $\mV\in\mbr{d_1\times d_2}$ and a vector $\vv\in\mbr{d_3}$, the notation  $\tV\!=\!\mV\kronstack\vv$ produces a tensor $\tV\!\in\!\mbr{d_1\times d_2\times d_3}$ where the $p$-th slice of $\tV$ is given by $\mV v_p$, $v_p$ being the $p$-th dimension of $\vv$. Symmetric third-order tensors of rank one are formed by the outer product of a vector $\vv\in\mbr{d}$ in modes two and three. That is, a rank-one $\tV\in\mbr{d\times d\times d}$ is obtained from $\vv$ as $\tV\!=\!({\kronstack}_3\vv\!\triangleq\!(\vv\vv^T)\kronstack\vv)$. Concatenation of $n$ tensors in mode $k$ is denoted as $\left[\tV_i\right]_{i\in\idx{n}}^{\oplus_k}$, where $\idx{n}$ is an index sequence $1,2,\cdots, n$. The Frobenius norm of tensor is given by  $\fnorm{\tV} = \sqrt{\sum_{i,j,k} \tVnb_{ijk}^2}$, where $\tVnb_{ijk}$ represents the $ijk$-th element of $\tV$. Similarly, the inner-product between two tensors $\tX$ and $\tY$ is given by $\left\langle\tX,\tY\right\rangle=\sum_{ijk}\tXnb_{ijk}\tYnb_{ijk}$.
%%of the sum of element-wise squares: $||\tV||_F\!=\!\Big({\sum\limits_{i,j,k}\tVnb_{ijk}^2}\Big)^{0.5}\!\!$ and an associated with it dot-product  $\left<\tX,\tY\right>\!=\!\sum\limits_{i,j,k}\!\tXnb_{ijk}\tYnb_{ijk}$.

\subsection{Kernel Linearization}
Let $X=\set{x_t\in\reals{n}: t=1,2,\ldots, T}$ be a set of data instances produced by some dynamic process at discrete time instances $t$. Let $K$ be a kernel matrix created from $X$, the $ij$-th element of which is: 
\begin{equation}
K_{ij} = \langle \phi(x_i), \phi(x_j)\rangle,
\end{equation}
where $\phi(\cdot)$ represents some kernelized similarity map. 

\begin{theorem}[Linear Expansion of the Gaussian Kernel]
	For all $x$ and $x'$ in $\reals{n}$ and $\sigma>0$, 
	\begin{align}
	&\psi\left(\frac{x-x'}{\sigma}\right) = \myexp{-\frac{1}{2\sigma^2}\enorm{x-x'}^2} \\\nonumber &\quad\quad=\left(\frac{2}{\pi\sigma^2}\right)^\frac{n}{2}\int_{z\in\reals{n}}\!\!\! \myexp{-\frac{1}{\sigma^2}\enorm{x-z}^2}\myexp{-\frac{1}{\sigma^2}\enorm{x'-z}^2} dz\\
	\label{eq:gk_exp}
	\end{align}
\end{theorem}
\begin{proof}
See~\cite{jebara_prodkers,ckn}.
\end{proof}

We can approximate the linearized kernel by choosing a finite set of pivots $z\in\pivotset=\set{z_1, z_2,\cdots, z_K;\sigma_1,\sigma_2,\cdots, \sigma_K}$\footnote{To simplify the notation, we assume that the pivot set includes the bandwidths $\sigma$ associated with each pivot as well.}. Using $\pivotset$, we can rewrite~\eqref{eq:gk_exp} as: 
\begin{equation}
\psi(x-x') \approx \langle\ \tilde{\phi}(x; \pivotset),\ \tilde{\phi}(x'; \pivotset)\ \rangle
\end{equation}
where 
\begin{equation}
	\tilde{\phi}(x; \pivotset)=\left(\,e^{-\frac{1}{\sigma_1^2} \enorm{x-z_1}^2}, \ldots, e^{-\frac{1}{\sigma_K^2} \enorm{x-z_K}^2}\,\right).
	\label{eq:feature_map}
\end{equation}
We call $\tilde{\phi}(x;\pivotset)$ as the approximate feature map for the input data point $x$. 

In the sequel, we linearize the constituent kernels as this allows us to obtain linear maps, which results in favourable computational complexities, i.e., we avoid computing a kernel explicitly between tens of thousands of video sequences~\cite{tensor_eccv}. Also, see~\cite{bo2009efficient,williams2001using} for connections of our method with the Nystr\"{o}m approximation and random Fourier features~\cite{rahimi2007random}.
\section{Proposed Method}
\label{sec:method}
In this section, we first describe our overall CNN framework based on the popular two-stream deep model for action recognition~\cite{simonyan2014two}. This precedes exposition of our higher-order pooling framework, in which we introduce appropriate kernels for the task in hand. We also describe a setup for learning the parameters of the kernel maps.

\begin{figure*}[ht]
	\centering
	\includegraphics[width=17cm]{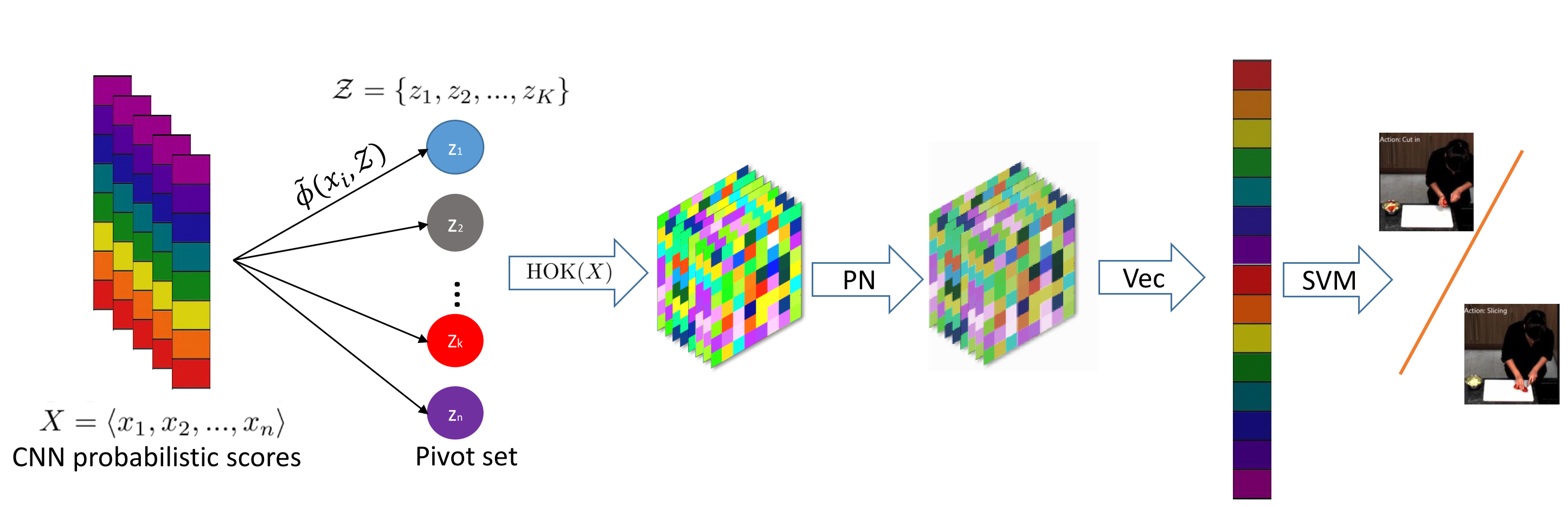}
	\caption{Our overall Higher-order Kernel (HOK) descriptor generation and classification pipeline. PN stands for eigen power-normalization.}
	\label{fig:pipeline}
\end{figure*}

\subsection{Problem Formulation}
Let $\seqset=\set{\seq_1, \seq_2, \ldots, \seq_N}$ be a set of video sequences, each sequence belonging to one of $M$ action classes with labels from the set $\labels=\set{\ell_1,\ell_2,\ldots, \ell_M}$. Let $\seq=\left<f_1, f_2,\ldots, f_n\right>$ be a sequence of frames of $\seq\in\seqset$, and let $\allframesset=\bigcup_{\seq\in\seqset}\set{f_i \mid f_i\in\seq}$. In the action recognition setup, our goal is to find a mapping from any given sequence to its ground truth label. Assume we have trained frame-level action classifiers for each class and that these classifiers cannot see all the frames in a sequence together. Suppose $\classifier_m:\allframesset\rightarrow [0,1]$ is one such classifier trained to produce a confidence score for an input frame to belong to the $m$-th class. Since a single frame is unlikely to represent well the entire sequence, the classifier $\classifier_m$ is inaccurate at determining the action at the sequence level. However, our hypothesis is that a combination of the predictions from all the classifiers across all the frames in a sequence could capture discriminative properties of the action and could improve recognition. In the sequel, we explore this possibility in the context of higher-order tensor descriptors.

\subsection{Correlations between Classifier Predictions}
Using the notations defined above, let $\left<f_1, f_2,\ldots, f_n\right>, f_i\in\seq$ denote a sequence corresponding to each frame and let $\classifier_m(f_i)$ denote the probability that a classifier trained for the $m$-th action class predicts $f_i$ to belong to class $\ell_m$. Then,
%
%\begin{align}
%\theta_m &= \frac{1}{\sqrt{n}}\left[\classifier_m(f_1), \classifier_m(f_2), \ldots, \classifier_m(f_n)\right]^T \in \reals{n}_+
%\end{align}
\begin{align}
\hatclassifier(f_i) &= \left[\classifier_1(f_i), \classifier_2(f_i),\cdots, \classifier_M(f_i)\right]^T,
\end{align}
% denotes the temporal probabilistic evolution of the confidence of the $m$-th classifier for the frames in the sequence.
denotes the class confidence vector for frame $i$. As described earlier, we assume that there exists a strong correlation between the confidences of the classifiers across time (temporal evolution of the classifier scores) for the frames from similar sequences; i.e., frames that are confused between different classifiers should be confused in a similar way for different sequences. To capture these correlations between the classifier predictions, we propose to use a kernel formulation on the scores from sequences, the $ij$-th entry of this kernel is as follows:
\begin{equation}
K(S_i,S_j)\!\!=\!\!\frac{1}{\Lambda}\!\!\sum_{t,u=1}^{n}\left[ \zeta_1\psi_1\!\!\left(\!\!\frac{\hatclassifier(f_{t})\!\!-\!\!\hatclassifier(f_{u})}{\sigma_F}\!\!\right)\!\!+\!\zeta_2\psi_2\left(\!\!\frac{t-u}{\sigma_T}\!\!\right)\right]^r\!\!\!,
\label{eq:7}
\end{equation}
where $\psi_1$ and $\psi_2$ are two RBF kernels. The kernel function $\psi_1$ captures the similarity between the two classifier scores at timesteps $t$ and $u$, while the kernel $\psi_2$ puts a smoothing on the length of the interval $[t,u]$. A small bandwidth $\sigma_T$ will demand the two classifier scores be strongly correlated at the respective time instances, while a larger $\sigma_T$ allows some variance (and hence more robustness) in capturing these correlations. \comment{We assume both $\psi_1$ and $\psi_2$ are RBF kernels so that the theory introduced in Section. XX can be directly used.} In the following, we look at linearizing the kernel in~\eqref{eq:7} for generating higher-order Kernel (HOK) descriptors. The parameter $r$ captures the order statistics of the kernel, as will be clear in the next section; $\zeta_1$ and $\zeta_2$ are weights associated with the kernels, and we assume $\zeta_1,\zeta_2>0, \zeta_1+\zeta_2=1$. $\Lambda$ is the normalization constant associated with the kernel linearization (see~\eqref{eq:gk_exp}). Note that we assume all the sequences are of the same length $n$ in the kernel formulation. This is a mere technicality to simplify our derivations. As will be seen in the next section, our HOK descriptor depends only on the length of one sequence (see Definition~\ref{def:HOK} below).

\subsection{Higher-order Kernel Descriptors}
\label{sec:HOK}
The following easily verifiable result \cite{me_tensor} will be handy in understanding our derivations.
\begin{proposition}
\label{prop:1}
Suppose $a,b\in\reals{d}$ are two arbitrary vectors, then for an ordinal $r>0$
\begin{equation}
\left(a^Tb\right)^{r} = \left\langle a^{\op{r}}, b^{\op{r}}\right\rangle.
\label{eq:3}
\end{equation}
\end{proposition}

For simplifying the notations, we assume $x^i_t = P(f^i_t)$, the score vector for frame $f_t$ in $S_i$. Further, suppose we have a set of pivots $\pivotset_F$ and $\pivotset_T$ for the classifier scores and the time steps, respectively. Then, applying the kernel linearization in~\eqref{eq:feature_map} to~\eqref{eq:7} using these pivots, we can rewrite each kernel as:
\begin{align}
 \psi_1\left(\frac{x^i_t - x^j_u}{\sigma_F}\right) &\approx \sum_{z_F\in\pivotset_F}\phi_F(x^i_t; z_F)\phi_F(x^j_u; z_F) \\\nonumber
			 &= \Phi_F(x^i_t; \pivotset_F)^T\Phi_F(x^j_u;\pivotset_F)\\
 \psi_2\left(\frac{t - u}{\sigma_T}\right) &\approx \sum_{z_T\in\pivotset_T}\phi_T(t; z_T)\phi_T(u; z_T) \\\nonumber
	 & = \Phi_T(t; \pivotset_T)^T\Phi_T(u;\pivotset_T).
 \label{eq:8}
\end{align}
Substituting~\eqref{eq:8} into~\eqref{eq:7}, we have:

\begin{align}
K(S_i,S_j) \approx \frac{1}{\Lambda}&\sum_{t,u=1}^n\left[\zeta_1\Phi_F(x^i_t; \pivotset_F)^T\Phi_F(x^j_u;\pivotset_F) + \right.\nonumber\\
		 &\quad\left.\zeta_2\Phi_T(t; \pivotset_T)^T\Phi_T(u;\pivotset_T)\right]^r
\end{align}
\begin{align}
=&\frac{1}{\Lambda}\sum_{t,u=1}^n\left\langle\left[\!\!\!\begin{array}{c}\sqrt{\zeta_1}\Phi_F(x^i_t; \pivotset_F)\\ \sqrt{\zeta_2}\Phi_T(t; \pivotset_T)\end{array}\!\!\!\right]^{\op{r}}\!\!\!, \left[\!\!\!\begin{array}{c}\sqrt{\zeta_1}\Phi_F(x^j_u;\pivotset_F)\\ \sqrt{\zeta_2}\Phi_T(u;\pivotset_T)\end{array}\!\!\!\right]^{\op{r}}\right\rangle,
\label{eq:10}
\end{align}
where we applied Proposition~\ref{prop:1} to~\eqref{eq:8}. As each component in the inner product in~\eqref{eq:10} is independent in the respective temporal indexes, we can carry the summation inside the terms leading to:
\begin{align}
\eqref{eq:10}\rightarrow&\left\langle\frac{1}{\sqrt{\Lambda}}\sum_{t=1}^n\left[\begin{array}{c}\sqrt{\zeta_1}\Phi_F(x^i_t; \pivotset_F)\\ \sqrt{\zeta_2}\Phi_T(t; \pivotset_T)\end{array}\right]^{\op{r}},\right.\nonumber\\
&\quad\quad\left. \frac{1}{\sqrt{\Lambda}}\sum_{u=1}^n\left[\begin{array}{c}\sqrt{\zeta_1}\Phi_F(x^j_u;\pivotset_F)\\ \sqrt{\zeta_2}\Phi_T(u;\pivotset_T)\end{array}\right]^{\op{r}}\right\rangle.
\label{eq:11}
\end{align}

%For example, we could impose a temporal weighting to the correlations so that they adhere to some prior conditions. To this end, assuming a kenrel $\psi_1$ on the classifier scores, and another kenrel $\psi_2$ on the temporal frame indicies, we introduce the following extension by including the temporal ordering of the sequences into the kernel as following:
%
%\begin{equation}
%K'_{jk} = \sum_{i=1}^n \psi_1\left(\classifier_j(f_i) - \classifier_k(f_i)\right)\psi_2()
%\end{equation}
%where $\sigma_f$ and $\sigma_T$ are the bandwidths of the respective RBF kernels.

Using these derivations, now we are ready to formally define our Higher-order Kernel (HOK) descriptor as follows:
\begin{definition}[HOK]
\label{def:HOK}
Let $X=\langle x_1, x_2,\cdots, x_n\rangle, x_i\in [0,1]^d$ are the probability scores from $d$ classifiers for the $n$ frames in a sequence. Then, we define the $r$-th order HOK-descriptor for  $X$ as:
\begin{equation}
\HOK(X) = \frac{1}{\sqrt{\Lambda}}\sum_{u=1}^n\left[\begin{array}{c}\sqrt{\zeta_1} \Phi_F (x_u; \pivotset_F)\\ \sqrt{\zeta_2} \Phi_T(u; \pivotset_T)\end{array}\right]^{\op{r}}
\label{eq:12}
\end{equation}
for pivot sets $\pivotset_F\subset\simplex{d}$ and $\pivotset_T\subset \reals{}$ for the classifier scores and the temporal instances respectively. Further, $\zeta_1,\zeta_2\in[0,1]$ such that $\zeta_1+\zeta_2=1$, and $\Lambda$ is a suitable normalization.
\end{definition}
 
Once the HOK tensor is generated for a sequence, we vectorize it to be used in a linear classifier for training against action labels. As can be verified (i.e., see \cite{me_tensor}), the HOK tensor will be super-symmetric, and thus removing the symmetric entries, the dimensionality of this descriptor is $\left(\begin{array}{c}|\pivotset_F|+|\pivotset_T|+r-1\\ r\end{array}\right)$. In the sequel, we use $r=3$ as a trade-off between performance and the descriptor size. Figure~\ref{fig:pipeline} illustrates our overall HOK generation and classification framework.

\subsection{Power Normalization}
It is often found that using a non-linear operator on higher-order tensors leads to significantly better performance~\cite{me_ATN}. For example, for BOW, unit-normalization is known to avoid the impact of background features, while taking the feature square-root reduces~\emph{burstiness}. Motivated by these observations, we may incorporate such non-linearities to the HOK descriptors as well. As these are higher-order tensors, we apply the following scheme based on the higher-order SVD decomposition of the tensor~\cite{me_tensor,sparse_tensor_cvpr,tensor_eccv}. Let $H$ denote the $\HOK(X)$, then
\begin{align}
&[\Sigma, U_1, U_2,\cdots, U_r] = \HOSVD(H)\\
&\hat{H} =\sign(\Sigma)\left|\Sigma\right|^{\alpha}\times_{i=1}^r U_i %\left(\cdots\left(\sign(\Sigma)\left|\Sigma\right|^{\alpha}\otimes_1 U_1\right)\cdots \otimes_{r} U_{r}\right),
\end{align} 
where the $U$'s are the orthonormal matrices (which are all the same in our case) associated with the Tucker decomposition~\cite{lathauwer_hosvd} and $\Sigma$ is the core tensor.  Note that, unlike the usual SVD operation for matrices, the core tensor in HOSVD is generally not diagonal. Refer to the notations in Section~\ref{sec:prelims} for the definition of $\times_{i=1}^n$. We use $\hat{H}$ in training the linear classifiers after vetorization. The power-normalization parameter $\alpha$ is selected via cross-validation.

\subsection{Computational Complexity}
\label{sec:complexity}
For $M$ classes, $n$ frames per sequence, $\card{\pivotset}$ pivots, and tensor-order $r$, generating HOK takes $\bigoh(n(M+\card{\pivotset}^r))$. As HOK is super-symmetric, using truncated SVD for a rank $k<<\card{\pivotset}$, HOSVD takes only $\bigoh(k^2|Z|)$ time. See~\cite{tensor_eccv} for more details.

\subsection{Learning HOK Parameters}
An important step in using the HOK descriptor is to find appropriate pivot sets $\pivotset_F$ and $\pivotset_T$. Given that the temporal pivots are uni-dimensional, we select them to be equally-spaced along the time axis after normalizing the temporal indexes to $[0,1]$. For $\pivotset_F$, which operate on the classifier scores, that can be high-dimensional, we propose to use an expectation-maximization algorithm. This choice is motivated by the fact that the entries for  $\zeta_1\Phi_F(x_u; \pivotset_F)$ in~\eqref{eq:10} are essentially computing a soft-similarity between  the classifier score vectors for every frame against the pivots through a Gaussian kernel. Thus modeling the problem in a soft-assignment setup using a Gaussian mixture model is natural, the parameters (the mean and the variance) are learned using the EM algorithm; these parameters are used as the pivot set. Other parameters in the model, such as $\zeta$ are computed using cross-validation. The normalization factor $\Lambda$ is chosen to be $n^2$ where $n$ is sequence length.
\section{Experiments and Results}
\label{sec:expts}
This section provides experimental evidence of the usefulness of our proposed pooling scheme for fine-grained action recognition. We verify this on two popular benchmark datasets for this task, namely, (i) the MPII cooking activities dataset~\cite{rohrbach2012database}, and (ii) the JHMDB dataset~\cite{jhuang2013towards}. Note that we use a VGG-16 model~\cite{Chatfield14} for the two stream architecture for both datasets, which is pre-trained on ImageNet for object recognition and fine-tuned on the respective action datasets. 

\subsection{Datasets}
\paragraph{MPII Cooking Acitivies Dataset~\cite{rohrbach2012database}:}consists of high-resolution videos of cooking activities recorded by a stationary camera. The dataset consists of videos of people cooking various dishes; each video contains a single person cooking a dish, and overall there are 12 such videos in the dataset. There are 64 distinct activities spread across 3748 video clips and one background activity (1861 clips). The activities range from coarse subject motions such as~\emph{moving from X to Y},~\emph{opening refrigerator}, etc., to fine-grained actions such as~\emph{peel},~\emph{slice},~\emph{cut apart}, etc. %This dataset is challenging due to several reasons, namely (i) the action classes are heavily unbalanced -- there are certain activities that have only about 1K frames over the entire dataset, (ii) there is significant intra-class differenes as the participants are only asked to prepare one from set of 14 possible dishes, and are allowed to cook in their own styles. Further, there are neither any annotations of objects in the scene, nor the tools used for actions are clearly visible (such as spice folder, knife, etc.). 

\paragraph{JHMDB Dataset~\cite{jhuang2013towards}:} is a subset of the larger HMDB dataset~\cite{kuehne2011hmdb}, but contains videos in which the human limbs can be clearly visible. The dataset contains 21 action categories such as \emph{brush hair}, ~\emph{pick},~\emph{pour},~\emph{push}, etc. Unlike the MPII cooking activities dataset, the videos in JHMDB dataset are low resolution. Each video clip is a few seconds long. There are a total of 968 videos which are mostly downloaded from the internet. % and contains significant changes in lighting, viewpoints, slight camera motions, and inter/intra-class variability, that makes the recognition task challenging.

 %trim={<left> <lower> <right> <upper>}
\begin{figure*}
	\centering
	\subfigure[]{\includegraphics[width=5.5cm,trim={0 7cm 0 10cm}]{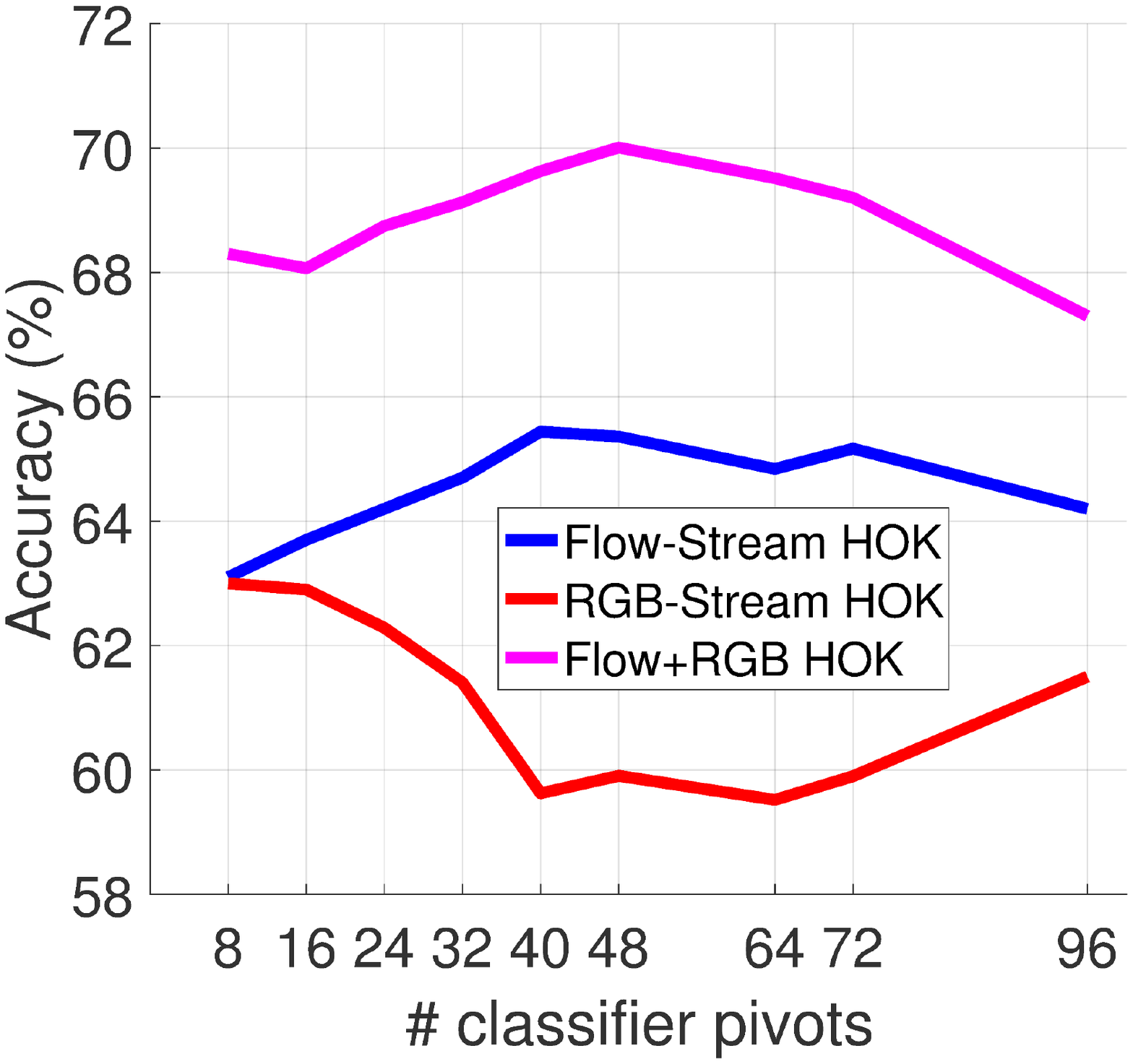}\label{fig:a}}
	\subfigure[]{\includegraphics[width=5.5cm,trim={0 7cm 0 10cm}]{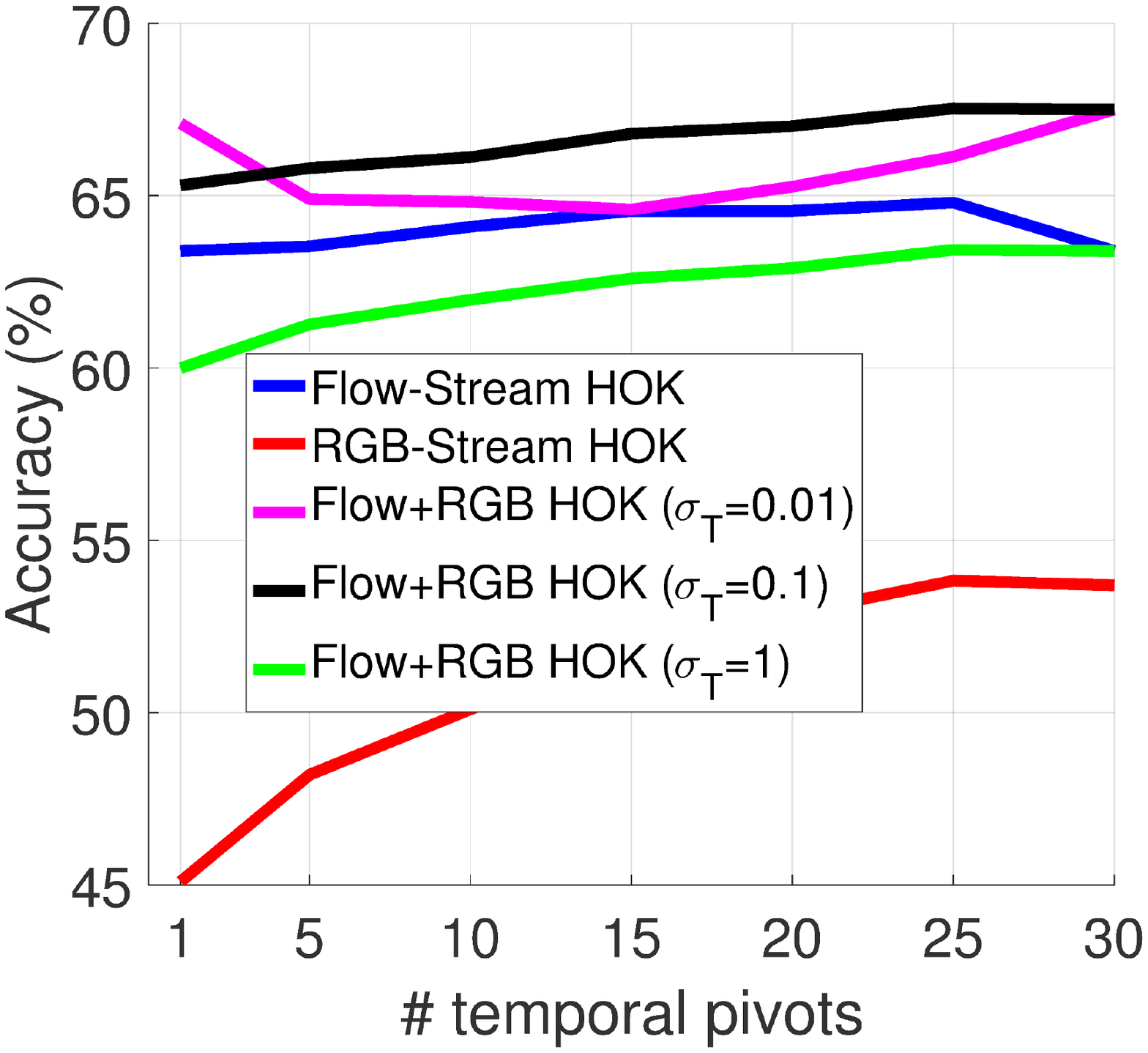}\label{fig:b}}
	\subfigure[]{\includegraphics[width=5.5cm,trim={0 7cm 0 10cm}]{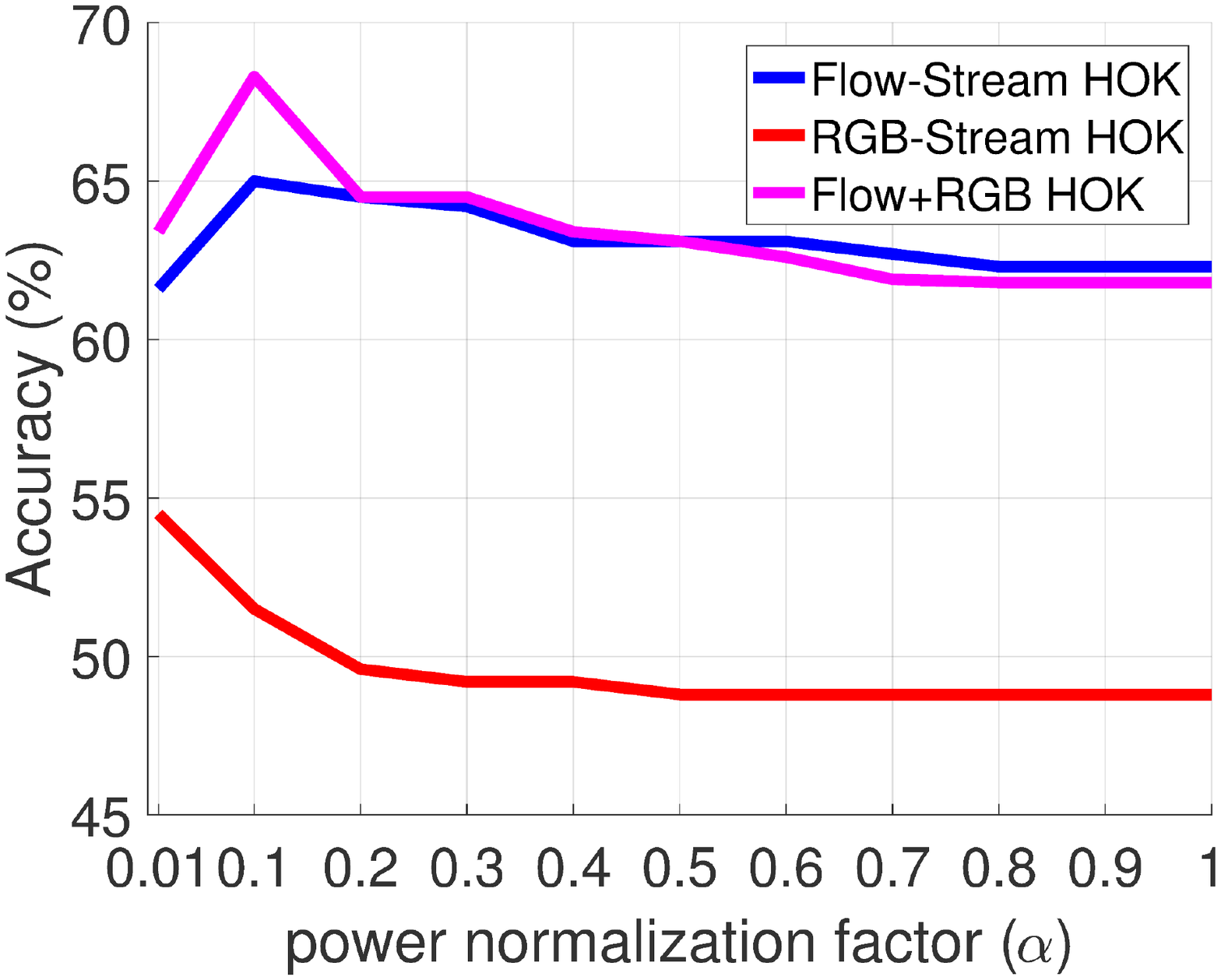}\label{fig:c}}
	\caption{Analysis of the influence of various hyper-parameters on the action recognition accuracy. The numbers are computed on the split1 of the JHMDB dataset, which consists of 21 ground truth action classes.}
	\label{fig:analysis}
\end{figure*}

\subsection{Evaluation Protocols} We follow the standard protocols suggested in the original publications that introduced these datasets. Thus, we use the mean average precision (mAP) over 7-fold cross-validation for the MPII dataset, while we use the mean average accuracy over 3-fold cross-validation for the JHMDB dataset. For the former, we use the evaluation code published with the dataset.

\subsection{Preprocessing} As the original MPII cooking videos are of very high resolution, while the activities happen only at certain parts of the scene, we used a frame difference scheme to estimate a window of the scene to localize the action. Precisely, for every sequence, we first convert the frames to half their sizes, followed by frame-differencing, dilation, smoothing, and connected component analysis. This results in a binary image for every frame; which are then combined across the sequence and a binary mask is generated for the entire sequence. We use the largest bounding box containing all the connected components in this binary mask as the region of the action, and crops the video to this box. Such cropped frames are then resized to $224\times 224$ size and used to train the VGG networks. To compute optical flow, we used the Brox implementation~\cite{brox2011large}. Each flow image is rescaled to 0--255 and saved as a JPEG image for storage efficiency as described in~\cite{simonyan2014two}. For the JHMDB dataset, the frames are already in low resolution. Thus, we directly use them in the CNN after resizing to the expected input sizes.

\subsection{CNN Training} The two streams of the CNN are trained separately on the respective input modalities against a softmax cross-entropy loss. We used the sequences from the training set of the MPII cooking activities dataset for training the CNNs (1992 sequences) and used those from the provided validation set (615 sequences) to check for overfitting. For the JHMDB, we used 50\% of the training set for fine-tuning the CNNs of which 10\% is used as the validation set. We augmented the datasets using random crops, flips and slight rotations of the frames. While fine-tuning the CNNs (from a pre-trained imagenet model), we used a fixed learning rate of $10^{-4}$ and an input batch size of 50 frames. The CNN training was stopped as soon as the loss on the validation set started increasing, which happened in about 6K iterations for the appearance stream and 40K iterations for the flow stream. 

\subsection{HOK Parameter Learning}
As is clear from Def.~\ref{def:HOK}, there are a few hyper-parameters associated with the HOK descriptor. In this section, we systematically analyze the influence of these parameters to the overall classification performance of the descriptor. To this end, we use the JHMDB dataset split1. Specifically, we explore the effect of changes to (i) the number of classifier pivots $\pivotset_F$, (ii) the number of temporal pivots $\pivotset_T$, and (iii) that of the power-normalization factor $\alpha$. In Figure~\ref{fig:analysis}, we plot the classifier accuracy against each of these cases. Each experiment is repeated 5 times with different initializations (for the GMM) and the average accuracy is used for the plots. 

For (i), we fixed the number of temporal pivots to 5, with values $[0,0.25,0.5,0.75, 1]$ and fixed the $\sigma_T=0.1$. The classifier pivots $\pivotset_F$ and and their standard deviations $\sigma_F$ are found by learning GMM models with the prescribed number of Gaussian components. The mean and the diagonal variance from this learned model are then used as the pivot set and its variance respectively. As is clear from Figure~\ref{fig:a}, as the number of classifier pivots increases, the accuracy increases as well. However, beyond a certain number, the accuracy starts dropping. This is perhaps due to the sequences not containing sufficient number of frames to account for larger models. Note that the JHMDB sequences contain about 30-40 frames per sequence. We also note that the accuracy of Flow+RGB is significantly higher than either stream alone.

For (ii), we fix the number of classifier pivots at 48 (as is the best we found from Figure~\ref{fig:b}), and varied the number of temporal pivots from 1 to 30 in steps of 5. Similar to the classifier pivots, we find that increasing the number of temporal pivots is beneficial. Further, a larger $\sigma_T$ leads to a drop in accuracy, which implies that ordering of the probabilistic scores does play a role in the recognition of the activity.

For (iii), we fixed the number of classifier pivots at 48, and the number of temporal pivots to 5 (as described for (i) above). We varied $\alpha$ from 0.1 to 1 in steps of 0.1. We find that $\alpha$ closer to 0 is more promising, implying that there is significant influence of burstiness in the sequences. That is, reducing more the larger probabilistic co-occurrences (than those of the weak co-occurrences) in the tensor leads to better performance.

\subsection{Results}
In this section, we provide full experimental comparisons for the two datasets. Our main goal is to analyze the usefulness of higher-order pooling for action recognition. To this end, in Table~\ref{tab:higher_order_pooling}, we show the performance differences between using (i) the first-order statistics, (ii) the second-order statistics and our proposed third-order. For (i), we average the classifier soft-max scores as is usually done in late pooling~\cite{simonyan2014two}. For (ii), we use the second-order kernel matrix without pivoting. Specifically, for every sequence, let $x^i_{:}$ and $x^j_{:}$ denote the probabilistic evolution of probablistic scores for classifiers $i$ and $j$ respectively. Then, we compute a kernel matrix $K(x_{:}^i, x_{:}^j)=\myexp{-\sigma\enorm{x_{:}^i-x_{:}^j}^2}$. As this matrix is a positive definite object, we use log-Euclidean map of it (that is, the matrix logarithm; which is the asymptotic limit of $\alpha\rightarrow 0$ in power normalization) for embedding it in the Euclidean space~\cite{guo2013action}. This vector is then used for training. As is clear, this matrix captures the second-order statistics of actions. And for (iii), we use the proposed $\HOK$ descriptor as described in Definition~\ref{def:HOK}. As is clear from Table~\ref{tab:higher_order_pooling}, higher-order statistics leads to significant benefits on both the datasets and for both the input modalities (flow and RGB) and their combinations.

\begin{table}[]
	\centering
	\begin{tabular}{c|c|c}
		Action & Avg. Pool  & HOK. Pool  \\ 
		           & mAP (\%)  & mAP (\%) \\
		\hline
		Change Temperature & 15.1 & 57.5 \\
        Dry   & 27.7 & 50.2 \\
        Fill water from tap & 10.5 & 40.6 \\
        Open/close drawer & 25.2 & 65.1 \\ 
	\end{tabular}    
	\caption{An analysis of per-class action recognition accuracy when using average pooling and HOK pooling (the top classes corrected by HOK pooling).
	}
	\label{tab:comparison_table}
\end{table}

\subsection{Comparisons to the State of the Art}
In Tables~\ref{tab:mpii_soa} and~\ref{tab:jhmdb_soa}, we compare HOK descriptors to the state-of-the-art results on these two datasets. In this case, we combine the HOK descriptors from both the RGB and flow streams of the CNN. For the MPII dataset, we use 32 pivots for the classifier scores, and 5 equispaced pivots for the time steps, with $\sigma_T=0.1$. For the second-order tensor, we use a $\sigma=0.1$ for both datasets. We use the same setup for the JHMDB dataset, except that we use 48 pivots. The power normalization factor is set to 0.1. As is clear, although HOK by itself is not superior to other methods, when the second- and third-order statistics are combined (stacking their values into a vector), it demonstrates significant promise. For example, we see an improvement of 5--6\% against the recent method in~\cite{cheron2015p} that also uses a CNN. Further, we also find that when the higher-order statistics are combined with trajectory features, there is further improvement in accuracy, which results in a model that outperforms the state of the art. 

\subsection{Analysis of Results}
To gain insights into the performance benefits noted above, we conducted an analysis of the results on the MPII dataset. Table~\ref{tab:comparison_table}  lists the activities that are initially confused in average pooling, while corrected by HOK. Specifically, we find that activities such as \emph{Fill water from tap} and \emph{Open/close Drawer} which are originally confused with~\emph{Wash Objects} and~\emph{Take out from drawer} gets corrected using higher-order pooling. Note that these activities are inherently ambiguous, unless context and sub-actions are analyzed. This shows that our descriptor can effectively represent useful cues for recognition.

In Table~\ref{tab:higher_order_pooling} (column 1), we see that  the second-order tensor performs significantly better than HOK for the MPII dataset. We suspect this surprising behavior is due to the highly unbalanced number of frames in sequences in this dataset. For example, for classes such as~\emph{pull-out},~\emph{pour}, etc., that have only about 7 clips each of 50--90 frames, the second-order is about 30\% better than HOK in mAP, while for classes, such as~\emph{take spice holder}, having more than 25 videos, with 50--150 frames, HOK is about 10\% better than second-order. This suggests that the poor performance is perhaps due to the unreliable estimation of data statistics and that second- and third-order provide complementary cues, as also witnessed in Table~\ref{tab:mpii_soa}. For the JHMDB dataset, there are about 30 frames in all sequences and thus the statistics are more consistent. Another reason could be that unbalanced sequences may bias the GMM parameters, that form the pivots, to classes that have more frames. 

\begin{table}[]
	\centering
	\begin{tabular}{c|c|c}
		Experiment & \small{MPII}  & JHMDB\\
		           & mAP (\%)      & Mean Avg. Acc (\%)\\ 
		\hline
		RGB (avg.pool)     & 33.9  & 51.5\\
		Flow (avg.pool)     & 37.6  & 54.8\\
		RGB + Flow (avg.pool)& 38.1  & 55.9\\
		\hline
		RGB  (second-order) & 56.1 & 52.3\\
		Flow (second-order) & 61.3 & 60.4\\
		RGB + Flow (second-order) & 67.0 & 63.4\\
		\hline
		RGB (HOK) & 47.8 & 52.3\\
		Flow (HOK) & 55.4 &  58.2\\
		RGB + Flow (HOK) & 60.6 & 64.7
	\end{tabular}    
	\caption{Evaluation of the HOK descriptor on the output of each CNN stream and their fusion on the MPII (7-splits) and JHMDB datasets (3-splits). We also show the accuracy obtained via second-order pooling that uses the kernel matrix directly without linearization (see text for details).
	}
	\label{tab:higher_order_pooling}
\end{table}

\begin{table}[]
		\centering
		\begin{tabular}{c|c}
			Algorithm &  mAP(\%) \\
			\hline
			Holistic + Pose, CVPR'12 & 57.9 \\
			Video Darwin, CVPR'15    & 72.0 \\
			Interaction Part Mining, CVPR'15 & 72.4 \\
			P-CNN, ICCV'15 & 62.3 \\
			P-CNN + IDT-FV, ICCV'15 & 71.4 \\
			Semantic Features, CVPR'15 & 70.5 \\
			Hierarchical Mid-Level Actions, ICCV'15 & 66.8\\
			\hline
			HOK (ours) & 60.6 \\
			HOK + Second-order (ours) & 69.1 \\
			HOK + second-order + Trajectories & \textbf{73.1}\\
		\end{tabular}
		\caption{MPII Cooking Activities dataset (7-splits)}
        \label{tab:mpii_soa}
\end{table}
\begin{table}[]
		\centering
		\begin{tabular}{c|c}
			Algorithm &  Avg. Acc. (\%) \\
			\hline
			P-CNN, ICCV'15 & 61.1 \\
			P-CNN + IDT-FV, ICCV'15 & 72.2 \\
			Action Tubes, CVPR'15 & 62.5\\
			Stacked Fisher Vectors, ECCV'14 & 69.03\\
			IDT + FV, ICCV'13 & 62.8 \\
			\hline
			HOK (Ours) & 64.7\\
			HOK + second-order (Ours) & 66.8 \\
			HOK + second-order + IDT-FV & \textbf{73.3}
		\end{tabular}
		\caption{JHMBD Dataset (3-splits)}
        \label{tab:jhmdb_soa}
	%\caption*{Comparisons against the state of the art. $^*$uses ground truth pose.}
\end{table}

\section{Conclusion}
\label{sec:conclude}
In this paper, we presented a technique for higher-order pooling of CNN scores for the task of action recognition in videos. We showed how to use the idea of kernel linearization to generate a higher-order kernel descriptor, which can capture latent relationships between the CNN classifier scores. Our experimental analysis on two standard fine-grained action datasets clearly demonstrates that using higher-order relationships is beneficial for the task and leads to state-of-the-art performance. %An interesting future direction of this work is to train the CNN to learn the higher-order temporal relationships in an end-to-end manner. %Investigating even higher-orders of pooling is also an interesting direction, albiet the main challenge is how to reduce the computational burden.
\noindent\paragraph{{Acknowledgements:}}{This research was supported by the Australian Research Council (ARC) through the Centre of Excellence for Robotic Vision (CE140100016).

%%%%%%%%% BODY TEXT
{\small
\bibliographystyle{ieee}
\bibliography{wacv_fgar_bib}
}

\end{document}